\documentclass{article}

\usepackage[utf8]{inputenc} 
\usepackage[T1]{fontenc}    
\usepackage{hyperref}       
\usepackage{url}            
\usepackage{booktabs}       
\usepackage{amsfonts}       
\usepackage{nicefrac}       
\usepackage{microtype}      
\usepackage{xcolor}         

\usepackage{graphicx}
\usepackage{makecell}
\usepackage{subfigure}
\usepackage{epsfig}
\usepackage{url}
\usepackage{amsfonts}
\usepackage{tipa}
\usepackage{times}
\usepackage{mathrsfs}
\usepackage{amssymb}
\usepackage{amsthm}
\usepackage{amsmath}
\usepackage[ruled,vlined,linesnumbered]{algorithm2e}
\usepackage{algorithmicx}
\usepackage{algpseudocode}

\newtheorem{lemma}{Lemma}[section]
\newtheorem{theorem}{Theorem}[section]
\newtheorem{claim}{Claim}[section]
\newcommand{\vol}{\textrm{vol}}
\newcommand{\cost}{\textrm{cost}}

\title{An Information-theoretic Perspective of Hierarchical Clustering}

\author{Yicheng Pan \and Feng Zheng \and Bingchen Fan \footnote{
		Beihang University, Beijing, 100191, P. R.
		China. Email: \{yichengp, cenphon, fanbc\}@buaa.edu.cn.
		Correspondence: yichengp@buaa.edu.cn.}
\date{}
}

\begin{document}

\maketitle

\begin{abstract}
	A combinatorial cost function for hierarchical clustering was introduced by Dasgupta \cite{dasgupta2016cost}. It has been generalized by Cohen-Addad et al. \cite{cohen2019hierarchical} to a general form named admissible function. In this paper, we investigate hierarchical clustering from the \emph{information-theoretic} perspective and formulate a new objective function. We also establish the relationship between these two perspectives. In algorithmic aspect, we get rid of the traditional top-down and bottom-up frameworks, and propose a new one to stratify the \emph{sparsest} level of a cluster tree recursively in guide with our objective function. For practical use, our resulting cluster tree is not binary. Our algorithm called HCSE outputs a $k$-level cluster tree by a novel and interpretable mechanism to choose $k$ automatically without any hyper-parameter. Our experimental results on synthetic datasets show that HCSE has a great advantage in finding the intrinsic number of hierarchies, and the results on real datasets show that HCSE also achieves competitive costs over the popular algorithms LOUVAIN and HLP.
\end{abstract}

\section{Introduction} \label{sec:introduction}
Hierarchical clustering for graphs plays an important role in the structural analysis of a given data set. Understanding hierarchical structures on the levels of multiple granularities is fundamental in various disciplines including artificial intelligence, physics, biology, sociology, etc \cite{brown1992class,eisen1998cluster,gorban2008principal,culotta2007author}. Hierarchical clustering requires a cluster tree that represents a recursive partitioning of a graph into smaller clusters as the tree nodes get deeper. In this cluster tree, each leaf represents a graph node while each non-leaf node represents a cluster containing its descendant leaves. The root is the largest one containing all leaves.

During the last two decades, flat clustering has attracted great attentions, which breeds plenty of algorithms, such as $k$-means \cite{hartingan1979kmeans}, DBSCAN \cite{ester1996density}, spectral clustering \cite{alpert1995spectral}, and so on. From the combinatorial perspective, we have cost functions, i.e., modularity and centrality measures, to evaluate the quality of partition-based clustering. Therefore, community detection is usually formulated as an optimization problem to optimize these objectives. By contrast, no comparative cost function with a clear and reasonable combinatorial explanation was developed until Dasgupta \cite{dasgupta2016cost} introduced a cost function for cluster trees. In this definition, similarity or dissimilarity between data points is represented by weighted edges. Taking similarity scenario as an example, a cluster is a set of nodes with relatively denser intra-links compared with its inter-links, and in a good cluster tree, heavier edges tend to connect leaves whose lowest common ancestor is as deep as possible. This intuition leads to Dasgupta's cost function that is a weighted linear combination of the sizes of lowest common ancestors over all edges.

Motivated by Dasgupta's cost function, Cohen-Addad et al. \cite{cohen2019hierarchical} proposed the concept of admissible cost function. In their definition, the size of each lowest common ancestor in Dasgupta's cost function is generalized to be a function of the sizes of its left and right children. For all similarity graphs generated from a minimal ultrametric, a cluster tree achieves the minimum cost if and only if it is a generating tree that is a ``natural'' ground truth tree in an axiomatic sense therein. A necessary condition of admissibility of an objective function is that it achieves the same value for every cluster tree for a uniformly weighted clique that has no structure in common sense. However, any slight deviation of edge weights would generally separate the two end-points of a light edge on a high level of its optimal (similarity-based) cluster tree. Thus, it seems that admissible objective functions, which take Dasgupta's cost function as a specific form, ought to be an unchallenged criterion in evaluating cluster trees since they are formulated by an axiomatic approach.

However, an admissible cost function seems imperfect in practice. The arbitrariness of optima of cluster trees for cliques indicates that the division of each internal nodes on an optimal cluster tree totally neglects the \emph{balance} of its two children. Edge weight is the unique factor that decides the structure of optimal trees. But a balanced tree is commonly considered as an ideal candidate in hierarchical clustering compared to an unbalanced one. So even clustering for cliques, a balanced partition should be preferable for each internal node. At least, the height of an optimal cluster tree which is logarithm of graph size $n$ is intuitively more reasonable than that of a caterpillar shaped cluster tree whose height is $n-1$. Moreover, a simple proof would imply that the optimal cluster tree for any connected graphs is binary. This property is not always useful in practical use since a real system usually has its inherent number of hierarchies and a natural partition for each internal clusters. For instance, the natural levels of administrative division in a country is usually intrinsic, and it is not suitable to differentiate hierarchies for parallel cities in the same state. This structure cannot be obtained by simply minimizing admissible cost functions.

In this paper, we investigate the hierarchical clustering from the perspective of information theory. Our study is based on Li and Pan's structural information theory \cite{li2016structural} whose core concept named structural entropy measures the complexity of hierarchical networks. We formulate a new objective function from this point of view, which builds the bridge for combinatorial and information-theoretic perspectives for hierarchical clustering. For this cost function, the balance of cluster trees will be involved naturally as a factor just like we design optimal codes, for which the balance of probability over objects is fundamental in constructing an efficient coding tree. We also define cluster trees with a specific height, which is coincident with our cognition of natural clustering. For practical use, we develop a novel algorithm for natural hierarchical clustering for which the number of hierarchies can be determined automatically. The idea of our algorithm is essentially different from the popular recursive division or agglomeration framework. We formulate two basic operations called \emph{stretch} and \emph{compress} respectively on cluster trees to search for the sparsest level iteratively. Our algorithm HCSE terminates when a specific criterion that intuitively coincides with the natural hierarchies is met. Our extensive experiments on both synthetic and real datasets demonstrate that HCSE outperforms the present popular heuristic algorithms LOUVAIN \cite{blondel2008fast} and HLP \cite{rossi2020fast}. The latter two algorithms proceed simply by recursively invoking flat clustering algorithms based on modularity and label propagation, respectively. For both of them, the hierarchy number is solely determined by the round numbers when the algorithm terminates, for which the interpretability is quite poor. Our experimental results on synthetic datasets show that HCSE has a great advantage in finding the intrinsic number of hierarchies, and the results on real datasets show that HCSE achieves competitive costs over LOUVAIN and HLP.

We organize this paper as follows. The structural information theory and its relationship with combinatorial cost functions will be introduced in Section \ref{sec:cost_functions}, and the algorithm will be given in Section \ref{sec:algorithm}. The experiments and their results are presented in Section \ref{sec:experiments}. We conclude the paper in Section \ref{sec:conclusions}.

\section{A cost function from information-theoretic perspective}
\label{sec:cost_functions}

In this section, we introduce Li and Pan's structural information theory \cite{li2016structural} and the combinatorial cost functions of Dasgupta \cite{dasgupta2016cost} and Cohen-Addad et al. \cite{cohen2019hierarchical}. Then we propose a new cost function that is developed from structural information theory and establish the relationship between the information-theoretic and combinatorial perspectives.

\subsection{Notations}
\label{subsec:notations}

Let $G=(V,E,w)$ be an undirected weighted graph with a set of vertices $V$, a set of edges $E$ and a weight function $w:E\rightarrow\mathbb{R}^+$, where $\mathbb{R}^+$ denotes the set of all positive real numbers. An unweighted graph can be viewed as a weighted one whose weights are unit. For each vertex $u\in V$, denote by $d_u=\sum_{(u,v)\in E} w(u,v)$ the weighted degree of $u$.\footnote{From now on, whenever we say the degree of a vertex, we always refer to the weighted degree.} For a subset of vertices $S\subseteq V$, define the volume of $S$ to be the sum of degrees of vertices. We denote it by $\vol(S)=\sum_{u\in S} d_u$.

A cluster tree $T$ for graph $G$ is a rooted tree with $|V|$ leaves, each of which is labeled by a distinct vertex $v\in V$. Each non-leaf node on $T$ is labeled by a subset $S$ of $V$ that consists of all the leaves treating $S$ as ancestor. For each node $\alpha$ on $T$, denote by $\alpha^-$ the parent of $\alpha$. For each pair of leaves $u$ and $v$, denote by $u\vee v$ the least common ancestor (LCA) of them on $T$.

\subsection{Structural information and structural entropy}
\label{subsec:structural_information}

The idea of structural information is to encode a random walk with a certain rule by using a high-dimensional encoding system for a graph $G$. It is well known that a random walk, for which a neighbor is randomly chosen with probability proportional to edge weights, has a stationary distribution on vertices that is proportional to vertex degree.\footnote{For connected graphs, this stationary distribution is unique, but not for disconnected ones. Here, we consider this one for all graphs.} So to position a random walk under its stationary distribution, the amount of information needed is typically the Shannon's entropy, denoted by $$\mathcal{H}^{(1)}(G)=-\sum_{v\in V} \frac{d_v}{\vol(V)} \log \frac{d_v}{\vol(V)} \footnote{In this paper, the omitted base of logarithm is always $2$.}.$$ By Shannon's noiseless coding theorem, $\mathcal{H}^{(1)}(G)$ is the limit of average code length generated from the \emph{memoryless} source for one step of the random walk. However, dependence of locations may shorten the code length. For each level on cluster trees, the uncertainty of locations is measured by the entropy of the stationary distribution on the clusters of this level. Consider an encoding for every cluster, including the leaves. Each non-root node $\alpha$ is labeled by its order among the children of its parent $\alpha^-$. So the self-information of $\alpha$ within this local parent-children substructure is $-\log(\vol(\alpha)/\vol(\alpha^-))$, which is also roughly the length of Shannon code for $\alpha$ and its siblings. The codeword of $\alpha$ consists of the sequential labels of nodes along the unique path from the root (excluded) to itself (included). The key idea is as follows. For one step of the random walk from $u$ to $v$ in $G$, to indicate $v$, we omit from $v$'s codeword the longest common prefix of $u$ and $v$ that is exactly the codeword of $u\vee v$. This means that the random walk takes this step in the cluster $u\vee v$ (and also in $u\vee v$'s ancestors) and the uncertainty at this level may not be involved. Therefore, intuitively, a quality similarity-based cluster tree would trap the random walk with high frequency in the deep clusters that are far from the root, and long codeword of $u\vee v$ would be omitted. This shortens the average code length of the random walk. Note that we ignore the uniqueness of decoding since a practical design of codewords is not our purpose. We utilize this scheme to evaluate and differentiate hierarchical structures.

Then we formulate the above scheme and measure the average code length as follows. Given a weighted graph $G=(V,E,w)$ and a cluster tree $T$ for $G$, note that under the stationary distribution, the random walk takes one step out of a cluster $\alpha$ on $T$ with probability $g_\alpha/\vol(V)$, where $g_\alpha$ is the sum of weights of edges with exactly one end-point in $\alpha$. Therefore, the aforementioned uncertainty measured by the average code length is
\[\mathcal{H}^T (G)=-\sum_{\alpha\in T} \frac{g_\alpha}{\vol(V)} \log \frac{\vol(\alpha)}{\vol(\alpha^-)}.\footnote{For notational convenience, for the root $\lambda$ of $T$, set $\lambda^-=\lambda$. So the term for $\lambda$ in the summation is $0$.}\]
We call $\mathcal{H}^T(G)$ the \emph{structural entropy of $G$ on $T$}. We define the \emph{structural entropy of $G$} to be the minimum one among all cluster trees, denoted by
$\mathcal{H}(G)=\min_{T}\{\mathcal{H}^T (G)\}.$
Note that the structural entropy of $G$ on the trivial $1$-level cluster tree is consistent with the previously defined $\mathcal{H}^{(1)}(G)$. It doesn't have any non-trivial cluster.

\subsection{Combinatorial explanation of structural entropy}

The cost function of a cluster tree $T$ for graph $G=(V,E)$ introduced by Dasgupta \cite{dasgupta2016cost} is defined to be $c^T(G)=\sum_{(u,v)\in E} w(u,v) |u\vee v|$, where $|u\vee v|$ denotes the size of cluster $u\vee v$. The admissible cost function introduced by Cohen-Addad et al. \cite{cohen2019hierarchical} generalizes the term $|u\vee v|$ in the definition of $c^T(G)$ to be a general function $g(|u|,|v|)$, for which Dasgupta defined $g(x,y)=x+y$. For both definitions, the optimal hierarchical clustering of $G$ is in correspondence with a cluster tree of minimum cost in the combinatorial sense that heavy edges are cut as far down the tree as possible.
The following theorem establishes the relationship between structural entropy and this kind of combinatorial form of cost functions.

\begin{theorem} \label{thm:equvi_cost_func}
	For a weighted graph $G=(V,E,w)$, to minimize $H^T(G)$ (over $T$) is equivalent to minimize the cost function
	\begin{equation} \label{eqn:SE_cost_form}
		\cost^T(G)=\sum_{(u,v)\in E} w(u,v) \log\vol(u\vee v).
	\end{equation}
\end{theorem}

\begin{proof}
	Note that
	\begin{eqnarray*}
		H^T(G) &=& -\sum_{\alpha\in T} \frac{g_\alpha}{\vol(V)}\log\frac{\vol(\alpha)}{\vol(\alpha^-)}\\
		&=& -\sum_{\alpha\in T} \sum_{(u,v)\in g_\alpha} \frac{w(u,v)}{\vol(V)}\log\frac{\vol(\alpha)}{\vol(\alpha^-)}\\
		&=& -\sum_{(u,v)\in E} \left(\frac{w(u,v)}{\vol(V)} \sum_{\alpha:(u,v)\in g_\alpha} \log\frac{\vol(\alpha)}{\vol(\alpha^-)}\right).
	\end{eqnarray*}	
	For a single edge $(u,v)\in E$, all the terms $\log(\vol(\alpha)/\vol(\alpha^-))$ for leaf $u$ satisfying $(u,v)\in g_\alpha$ sum (over $\alpha$) up to $\log (d_u/\vol(u\vee v))$ along the unique path from $u$ to $u\vee v$. It is symmetric for $v$. Therefore, considering ordered pair $(u,v)\in E$,		
	\begin{eqnarray*}		
		H^T(G) &=& -\sum_{\text{ordered }(u,v)\in E} \frac{w(u,v)}{\vol(V)}\log\frac{d_u}{\vol(u\vee v)}\\
		&=& \frac{1}{\vol(V)} \left( -\sum_{u\in V} d_u \log d_u + \sum_{\text{ordered }(u,v)\in E} w(u,v) \log \vol(u\vee v) \right)\\
		&=& \frac{1}{\vol(V)} \left( -\sum_{u\in V} d_u \log d_u + 2\cdot\sum_{(u,v)\in E} w(u,v) \log \vol(u\vee v) \right).
	\end{eqnarray*}
	The second equality follows from the fact $\sum_{u\in V}d_u=\sum_{\text{ordered }(u,v)\in E} w(u,v)=\vol(V)$ and the last equality from the symmetry of $(u,v)$. Since the first summation is independent of $T$, to minimize $H^T(G)$ is equivalent to minimize $\sum_{\{u,v\}\in E} w(u,v)\log \vol(u\vee v)$.
\end{proof}

Theorem \ref{thm:equvi_cost_func} indicates that when we view $g$ as a function of vertices rather than of numbers and define $g(u,v)=\log \vol(u\vee v)$, the ``admissible'' function becomes equivalent to structural entropy in evaluating cluster trees, although it is not admissible any more.

So what is the difference between these two cost functions? As stated by Cohen-Addad et al. \cite{cohen2019hierarchical}, an important axiomatic hypothesis for admissible function, thus also for Dasgupta's cost function, is that the cost for every binary cluster tree of an unweighted clique is identical. This means that any binary tree for clustering on cliques is reasonable, which coincides with the common sense that structureless datasets can be organized hierarchically free. However, for structural entropy, the following theorem indicates that balanced organization is of importance even though for structureless dataset.

\begin{theorem} \label{thm:SE_for_cliques}
	For any positive integer $n$, let $K_n$ be the clique of $n$ vertices with identical weight on every edge. Then a cluster tree $T$ of $K_n$ achieves minimum structural entropy if and only if $T$ is a balanced binary tree, that is, the two children clusters of each sub-tree of $T$ have difference in size at most $1$.
\end{theorem}

The proof of Theorem \ref{thm:SE_for_cliques} is a bit technical, and we defer it to Appendix \ref{sec:proof_thm_2.2}. The intuition behind Theorem \ref{thm:SE_for_cliques} is that balanced codes are the most efficient encoding scheme for unrelated data. So the codewords of the random walk that jumps freely among clusters on each level of a cluster tree have the minimum average length if all the clusters on this level are in balance. This is able to be guaranteed exactly by a balanced cluster tree.

In the cost function (\ref{eqn:SE_cost_form}), which we call cost(SE) from now on, $\log\vol(u\vee v)$ is a concave function of the volume of $u\vee v$. In the context of regular graphs (e.g. cliques), replacing $\vol(u\vee v)$ by $|u\vee v|$ is equivalent for optimization. Dasgupta \cite{dasgupta2016cost} claimed that for the clique $K_4$ of four vertices, a balanced tree is preferable when replace $|u\vee v|$ by $g(|u\vee v|)$ for any strictly increasing concave function $g$ with $g(0)=0$. However, it is interesting to note that this generalization does not hold for all those concave functions. For example, it is easy to check that for $g(x)=1-e^{-x}$, the cluster tree of $K_6$ that achieves minimum cost partitions $K_6$ into $K_2$ and $K_4$ on the first level, rather than $K_3$ and $K_3$. Theorem \ref{thm:SE_for_cliques} shows that for all cliques, balanced trees are preferable when $g$ is a logarithmic function.

It is worth noting the admissible function introduced by Cohen-Addad et al. \cite{cohen2019hierarchical} is defined from the viewpoint that a generating tree $T$ of a similarity-based graph $G$ that is generated from a minimal ultrametric achieves the minimum cost. In this definition, the monotonicity of edge weights between clusters on each level from bottom to top on $T$, which is given by Cohen-Addad et al. \cite{cohen2019hierarchical} as a property of a ``natural'' ground-truth hierarchical clustering, is the unique factor when evaluating $T$. However, Theorem \ref{thm:SE_for_cliques} implies that for cost(SE), besides cluster weights, the balance of cluster trees is implicitly involved as another factor. Moreover, for cliques, the minimum cost should be achieved on every subtree, which makes an optimal cluster tree balanced everywhere. This optimal clustering for cliques is also robust in the sense that a slight perturbation to the minimal ultrametric, which can be considered as slight variations to the weights of a batch of edges, will not change the optimal cluster tree structure wildly due to the holdback force of balance.

\section{Our hierarchical clustering algorithm} \label{sec:algorithm}

In this section, we develop an algorithm to optimize cost(SE) (Eq. (\ref{eqn:SE_cost_form}), equivalent to optimizing structure entropy) and yield the associated cluster tree. At present, all existing algorithms for hierarchical clustering can be categorized into two frameworks: top-down division and bottom-up agglomeration \cite{cohen2019hierarchical}. The top-down division approach usually yields a binary tree by recursively dividing a cluster into two parts with a cut-related criterion. But a binary clustering tree is far from a practical one as we introduced in Section \ref{sec:introduction}. For practical use, bottom-up agglomeration that is also known as hierarchical agglomerative clustering (HAC) is commonly preferable. It constructs a cluster tree from leaves to the root recursively, during each round of which the newly generated clusters shrink into single vertices.

Our algorithm jumps out of these two frameworks. We establish a new one that stratifies the \emph{sparsest} level of a cluster tree recursively rather than in a sequential order. In general, in guide with cost(SE), we construct a $k+1$-level cluster tree from the previous $k$-level one, during which the level whose stratification makes the average local cost that is incorporated in a local reduced subgraph decrease most is differentiated into two levels. The process of stratification consists of two basic operations: \emph{stretch} and \emph{compression}. In stretch steps, given an internal node of a cluster tree, a local binary subtree is constructed by an agglomerative approach, while in compression steps, the paths that are overlength from the root to leaves on the binary tree is compressed by shrinking tree edges that make the cost reduce most. This framework can be collocated with any cost function.

Then we define the operations ``stretch'' and ``compression'' formally. Given a cluster tree $T$ for graph $G=(V,E)$, let $u$ be an internal node on $T$ and $v_1,v_2,\ldots,v_\ell$ be its children. We call this local parent-children structure to be a \emph{$u$-triangle} of $T$, denoted by $T_u$. These two operations are defined on $u$-triangles. Note that each child $v_i$ of $u$ is a cluster in $G$. We reduce $G$ by shrinking each $v_i$ to be a single vertex $v_i'$ while maintaining each intra-link and ignoring each internal edge of $v_i$. This reduction captures the connections of clusters at this level in the parent cluster $u$. The stretch operation proceeds in HAC approach for $u$-triangle. That is, initially, view each $v_i'$ as a cluster and recursively combine two clusters into a new one for which cost(SE) drops most. The sequence of combinations yields a binary subtree $T_u'$ rooted at $u$ which has $v_1,v_2,\ldots,v_\ell$ as leaves. Then the compression operation is proposed to reduce the height of $T_u'$ to be $2$. Let $\hat{E}(T')$ be the set of edges on $T'$ each of which appears on a path of length more than $2$ from the root of $T'$ to some leaf. Denote by $\Delta(e)$ for edge $e$ be the amount of structural entropy enhanced by the shrink of $e$. We pick from $\hat{E}(T_u')$ the edge $e$ with least $\Delta(e)$. Note that the compression of a tree edge makes the grandchildren of some internal node to be children, it must amplify the cost. The compression operation picks the least amplification. The process of stretch and compression is illustrated in Figure \ref{fig:stretch_compress} and stated in Algorithms \ref{alg:stretch} and \ref{alg:compress}, respectively.

\begin{figure}[ht]	
	\centering
	\includegraphics[scale=0.5]{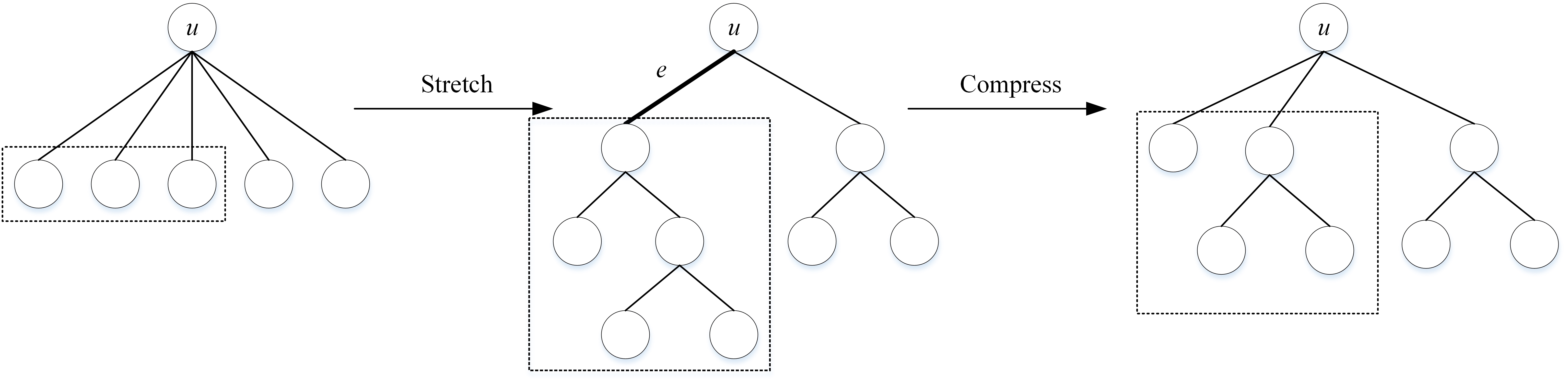}
	\caption{Illustration of stretch and compression for a $u$-triangle. A binary tree is constructed first by stretch, and then edge $e$ is compressed.}
	\label{fig:stretch_compress}
\end{figure}

\begin{algorithm} \label{alg:stretch}
	\caption{Stretch}
	\KwIn{a $u$-triangle $T_u$}
	\KwOut{a binary tree rooted at $u$}
	Let $\{v_1,v_2,\ldots,v_\ell\}$ be the set of leaves of $T_u$\;
	Compute $\eta(a,b)$ which is structural entropy reduced by merging siblings $a,b$ into a single cluster\;
	\For{$t\in [\ell-1]$}{
		$(\alpha,\beta) \gets \arg\max_{(a,b) \text{ are siblings}} \{\eta(a,b)\}$\;
		Add a new node $\gamma$\;
		$\gamma.parent \gets \alpha.parent$\;
		$\alpha.parent = \gamma$\;
		$\beta.parent = \gamma$\;
	}
	return $T_u$
\end{algorithm}

\begin{algorithm} \label{alg:compress}
	\caption{Compress}
	\KwIn{a binary tree $T$}
	\While{$T$'s height is more than $2$}{
		$e \gets \arg\min_{e'\in\hat{E}(T)} \{\Delta(e')\}$\;
		Denote $e=(u,v)$ where $u$ is the parent of $v$\;
		\For{$w\in v.children$}{
			$w.parent \gets u$\;
		}
	}	
\end{algorithm}

Then we define the sparsest level of a cluster tree $T$. Let $U_j$ be the set of $j$-level nodes on $T$, that is, $U_j$ is the set of nodes each of which has distance $j$ from $T$'s root. Suppose that the height of $T$ is $k$, then $U_0,U_1,\ldots,U_{k-1}$ is a partition for all internal nodes of $T$. For each internal node $u$, define $\mathcal{H}(u)=-\sum_{v:v^-=u} \frac{g_u}{\vol(V)} \log \frac{\vol(v)}{\vol(u)}$. Note that $\mathcal{H}(u)$ is the partial sum contributed by $u$ in $\mathcal{H}^T(G)$. After a ``stretch-and-compress'' round on $u$-triangle, denote by $\Delta\mathcal{H}(u)$ the structural entropy by which the new cluster tree reduces. Since the reconstruction of $u$-triangle stratifies cluster $u$, $\Delta\mathcal{H}(u)$ is always non-negative. Define the sparsity of $u$ to be
$\text{Spar}(u)=\frac{\Delta\mathcal{H}(u)}{\mathcal{H}(u)}$, which is the relative variation of structural entropy in cluster $u$. From the information-theoretic perspective, this means that the uncertainty of random walk can be measured locally in any internal cluster, which reflects the quality of clustering of this local area. At last, we define the \emph{sparsest level} of $T$ to be the $j$-th level such that the average sparsity of triangles rooted at nodes in $U_j$ is maximum, that is $\arg\max_j \{\overline{\text{Spar}}_j(T)\}$, where $\overline{\text{Spar}}_j(T)=\sum_{u\in U_j}\text{Spar(u)}/|U_j|$. Then the operation of stratification stretches and compresses on the sparsest level of $T$. This is illustrated in Figure \ref{fig:stratify}.

\begin{figure}[htbp]
\begin{center}	
	\subfigure[]{
		\begin{minipage}[t]{0.5\linewidth}
			\centering
			\includegraphics[width=6cm]{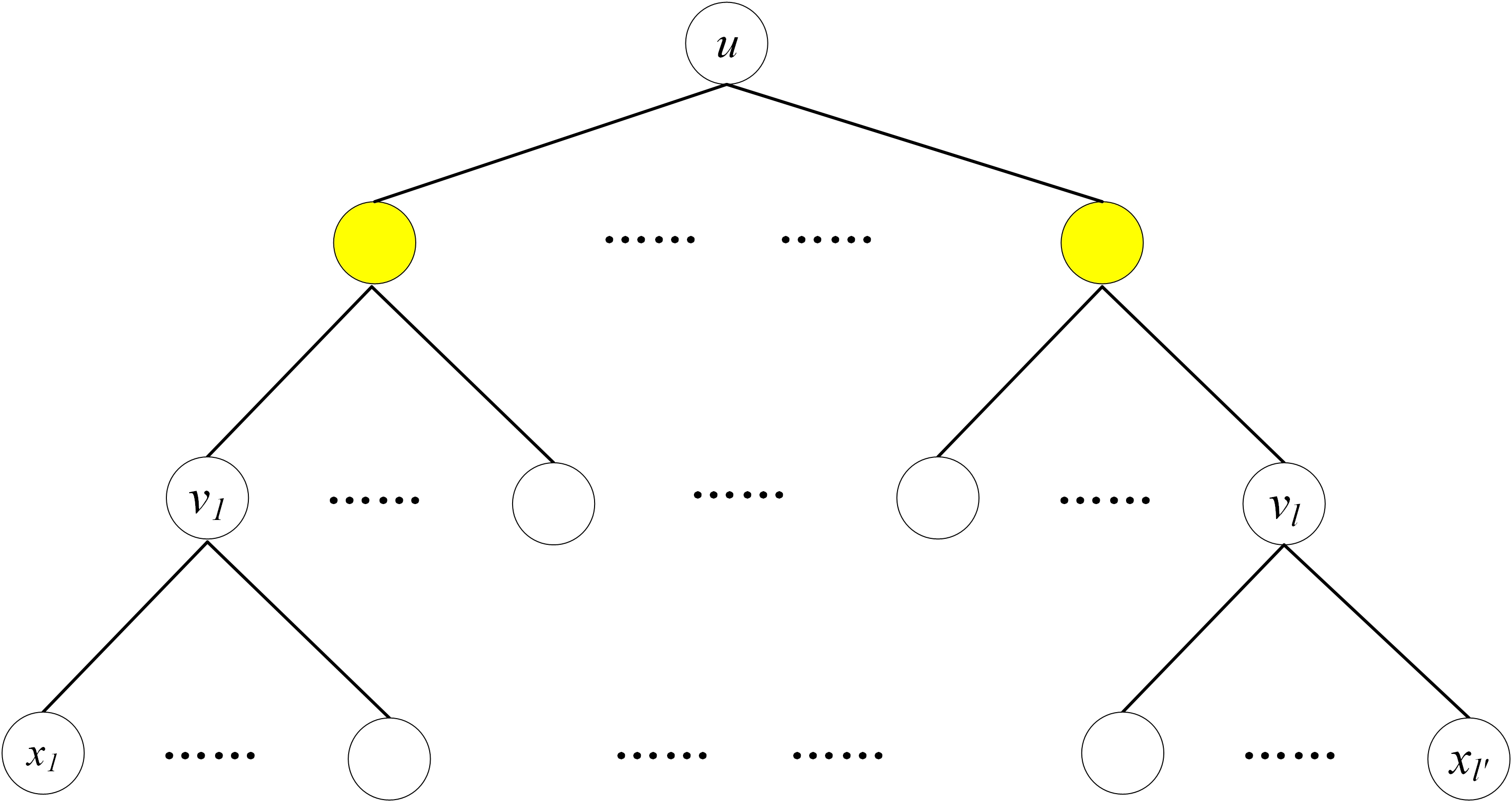}
			\label{fig:up}
		\end{minipage}%
	}%
	\subfigure[]{
		\begin{minipage}[t]{0.5\linewidth}
			\centering
			\includegraphics[width=6cm]{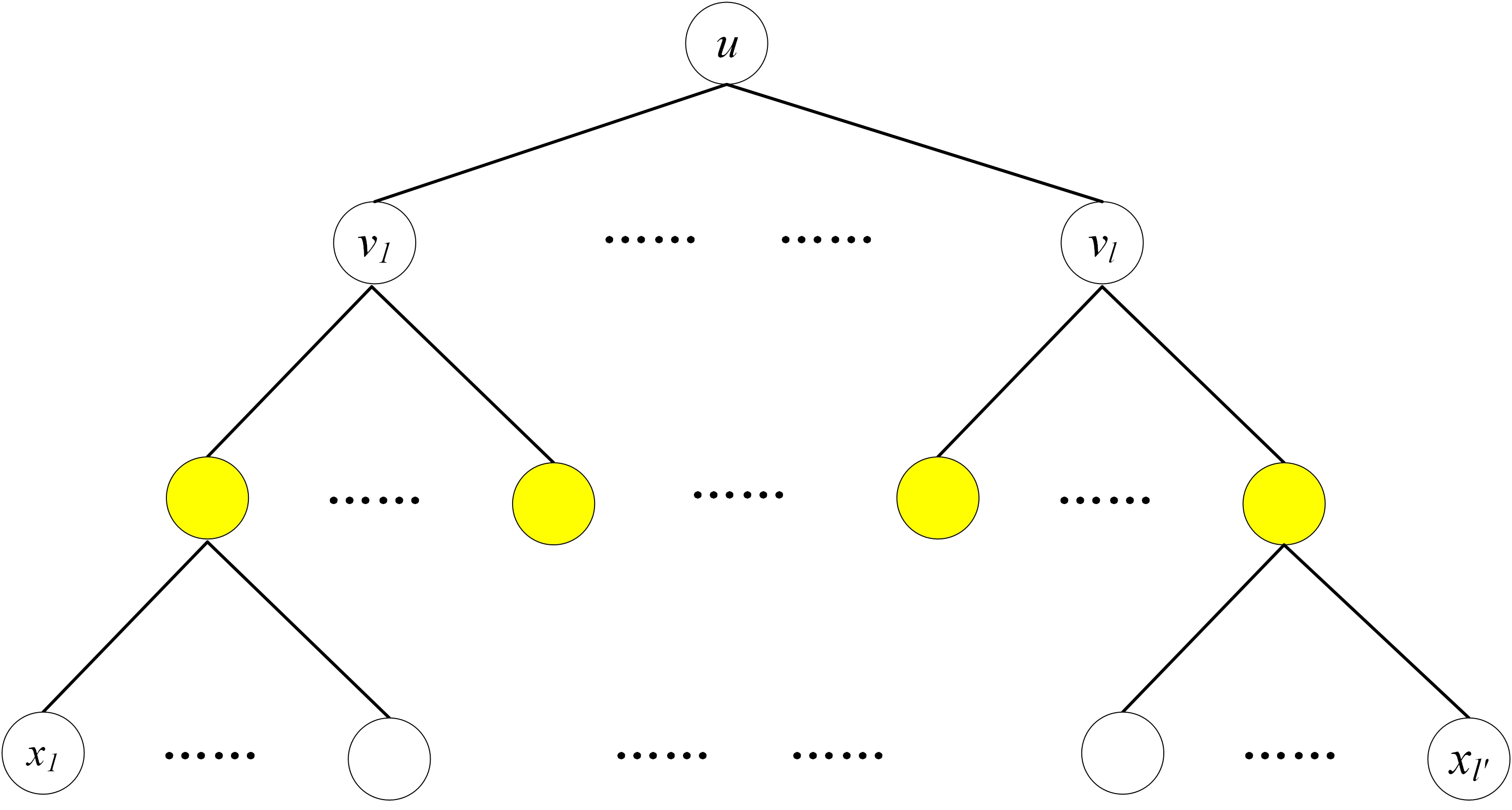}
			\label{fig:below}
		\end{minipage}%
	}%
	\caption{Illustration of stratification for a $2$-level cluster tree. The preference of (a) and (b) depends on the average sparsity of triangles at each level.}
	\label{fig:stratify}
\end{center}
\end{figure}

For a given positive integer $k$, to construct a cluster tree of height $k$ for graph $G$, we start from the trivial $1$-level cluster tree that involves all vertices of $G$ as leaves. Then we do not stop stratifying at the sparsest level recursively until a $k$-level cluster tree is obtained. This process is described in Algorithm \ref{alg:k-HCSE}.

\begin{algorithm} \label{alg:k-HCSE}
	\caption{$k$-Hierarchical clustering based on structural entropy ($k$-HCSE)}
	\KwIn{a graph $G = (V,E)$, $k\in\mathbb{Z}^+$}
	\KwOut{a $k$-level cluster tree $T$}
	Initialize $T$ to be the $1$-level cluster tree\;
	$h=\text{height(T)}$\;
	\While{$h<k$}{
		$j' \gets \arg\max_{j} \{\overline{\text{Spar}}_{j}(T)\}$; \quad // Find the sparsest level of $T$ (breaking ties arbitraily)\;
		\If{$\overline{\text{Spar}}_{j'}(T)=0$}{
			break;	\quad // No cost will be saved by any further clustering\;
		}
		\For{$u\in U_{j'}$}{
			$T_u \gets$ Stretch($u$-triangle $T_u$)\;
			Compress($T_u$)\;
			$h \gets h+1$\;
		}
		\For{$j\in [j'+1,h]$}{
			Update $U_j$\;
		}
	}
	return $T$
\end{algorithm}

To determine the height of the cluster tree automatically, we derive the natural clustering from the variation of sparsity on each level. Intuitively, a natural hierarchical cluster tree $T$ should have not only sparse boundary on clusters, but also low sparsity for triangles of $T$, which means that stratification within the reduced subgraphs corresponding to the triangles on the sparsest level make little sense. For this reason, we consider the inflection points of the sequence $\{\delta_t(\mathcal{H})\}_{t=1,2,\ldots}$, where $\delta_t(\mathcal{H})$ is the structural entropy by which the $t$-th round of stratification reduces. Formally, denote $\Delta_t\mathcal{H}=\delta_t(\mathcal{H})-\delta_{t-1}(\mathcal{H})$ for each $t\geq 2$. We say that $\Delta_t\mathcal{H}$ is an inflection point if both $\Delta_t\mathcal{H} \geq \Delta_{t-1}\mathcal{H}$ and $\Delta_t\mathcal{H} \geq \Delta_{t+1}\mathcal{H}$ hold. Our algorithm finds the least $t$ such that $\Delta_t\mathcal{H}$ is an inflection point and fix the height of the cluster tree to be $t$ (Note that after $t-1$ rounds of stratification, the number of levels is $t$). This process is described as Algorithm \ref{alg:HCSE}.

\begin{algorithm} \label{alg:HCSE}
	\caption{Hierarchical clustering based on structural entropy (HCSE)}
	\KwIn{a graph $G = (V,E)$}
	\KwOut{a cluster tree $T$}
	$t \gets 2$\;
	\While{$\Delta_t\mathcal{H} < \Delta_{t-1}\mathcal{H}$ or $\Delta_t\mathcal{H} < \Delta_{t+1}\mathcal{H}$}{
		\If{$\max_{j} \{\overline{\text{Spar}}_{j}(T)\}$=0}{
			break\;
		}
		$t \gets t+1$\;
	}
	return $t$-HCSE$(T)$
\end{algorithm}

\section{Experiments} \label{sec:experiments}

Our experiments are given both on synthetic networks generated from the Hierarchical Stochastic Block Model (HSBM) and on real datasets. We compare our algorithm HSE with the popular practical algorithms LOUVAIN \cite{blondel2008fast} and HLP \cite{rossi2020fast}. Both of these two algorithms construct a non-binary cluster tree with the same framework, that is, the hierarchies are formed from bottom to top one by one. In each round, they invoke different flat clustering algorithms, Modularity and Label Propagation, respectively. To avoid over-fitting to higher levels, which possibly results in under-fitting to lower levels, LOUVAIN admits a sequential input of vertices. Usually, to avert the worst-case trap, the order in which the vertices come is random, and so the resulting cluster tree depends on this order. HLP invokes the common LP algorithm recursively, and so it cannot be guaranteed to avoid under-fitting in each round. This can be seen in our experiments on synthetic datasets, for which these two algorithms usually miss a ground-truth level.

For real datasets, we do the comparative experiments on real networks. Some of them have (possibly hierarchical) ground truth, e.g., Amazon, while most of them do not have. We evaluate the resulting cluster trees for Amazon networks by Jaccard index, and for others without ground truth by both cost(SE) and Dasgupta's cost function cost(Das).

\subsection{HSBM}

The stochastic block model (SBM) is a type of probability generation model proposed based on the idea of random equivalence. The SBM of flat graphs needs to design the number of clusters, the number of vertices in each cluster, the probability of generating an edge for each pair of vertices in clusters, and the probability of generating an edge for each pair of vertices that are from different clusters. For HSBM, a ground truth of hierarchies should be presumed. Suppose that there are $m$ clusters at the bottom level. Then probability of generating edges is determined by a symmetric $m\times m$ matrix, in which the $(i,j)$-th entry is the probability of connecting each pair of vertices from cluster $i$ and $j$, respectively. Two clusters that have higher LCA on the ground-truth tree have lower probability.

Our experiments utilize $4$-level HSBM. For simplicity, let $\vec{p}=(p_0,p_1,p_2,p_3)$ be the probability vector for which $p_i$ is the probability of generating edges for vertex pairs whose LCA on the ground-truth cluster tree has depth $i$. Note that the $0$-depth node is the root. We compare the Normalized Mutual Information (NMI) at each level of the ground-truth cluster tree to those of three algorithms. Note that the randomness in LOUVAIN, and breaking-ties rule as well as convergence of HLP make different results, we choose the most effective strategy and pick the best results in five runs for both of them. Compared to their uncertainty, our algorithm HCSE yields stable results.

Table \ref{tab:HSBM_NMI} demonstrates the results in three groups of probabilities, for which the clarity of hierarchical structure gets higher one by one. Our algorithm HSE is always able to find the right number of levels, while LOUVAIN always misses the top level, and HLP misses the top level in two groups. The inflection points for choosing the intrinsic hierarchy number $t=4$ of hierarchies are demonstrated in Figure \ref{fig:inflection_points}.

\makeatletter
\newcommand\figcaption{\def\@captype{figure}\caption}
\newcommand\tabcaption{\def\@captype{table}\caption}
\makeatother

\begin{figure}[tb]
	\centering
	\begin{minipage}{0.42\textwidth}
		\centering
		\begin{tabular}{ccccc}
			\toprule
			&  $\vec{p}$  & HCSE     & HLP      & LOU  \\
			\midrule
			$p_2$ &  4.5E(-2)  & 0.89  & 0.79      & \textbf{0.92}\\
			$p_1$ &  1.5E(-3)  & \textbf{0.93} & 0.75       & 0.92 \\
			$p_0$ &  6E(-6)  & \textbf{0.62}       & 0.58       & \verb|--|\\
			\midrule
			$p_2$ &  5.5E(-2)  & 0.87  & \textbf{0.89}      & 0.89\\
			$p_1$ &  1.5E(-3)  & \textbf{0.95} & 0.87       & 0.87 \\
			$p_0$ &  4E(-6)  & \textbf{0.72}       & \verb|--|       & \verb|--| \\
			\midrule
			$p_2$ &  6.5E(-2)  & 0.96  & 0.95      & \textbf{0.99}\\
			$p_1$ &  4.5E(-3)  & 0.94 & 0.81       & \textbf{0.99} \\
			$p_0$ &  2.5E(-6)  & \textbf{0.80}       & \verb|--|       & \verb|--| \\
			\bottomrule
		\end{tabular}
		\tabcaption{\footnotesize NMI for three algorithms. Each dataset has $2,500$ vertices, and the cluster numbers at three levels are $5$, $25$ and $250$, respectively, for which the size of each cluster is accordingly
			generated at random. $p_3=0.9$ for each graph. ``$--$'' means the algorithm did not find this level.}
		\label{tab:HSBM_NMI}
	\end{minipage}
	\hspace{0.5in}
	\begin{minipage}[h]{0.42\linewidth}
		\centering
		\includegraphics[scale=0.4]{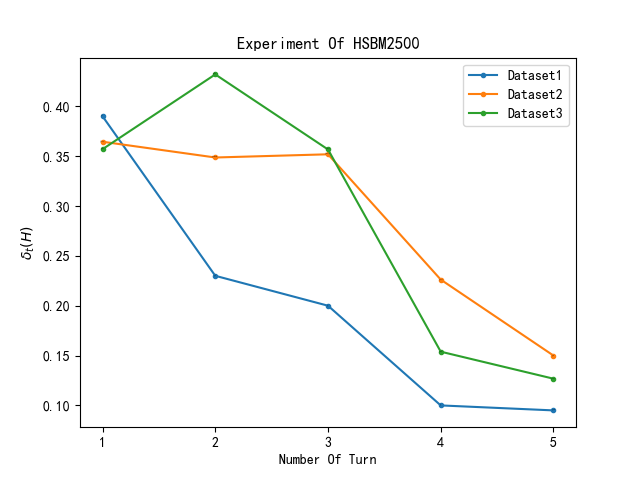}
		\figcaption{$\delta_t(\mathcal{H})$ variations for HCSE. It can be observed easily that the inflection points for all the three datasets appear on $t=4$, which is also the ground-truth number of hierarchies.}
		\label{fig:inflection_points}
	\end{minipage}
\end{figure}

\subsection{Real datasets}

First, we do our experiments on Amazon network \footnote{http://snap.stanford.edu/data/} for which the set of ground-truth clusters has been given. For two sets $A,B$, the \emph{Jaccard Index} of them is defined as $J(A,B)=|A\cap B|/|A\cup B|$. We pick the largest cluster which is a subgraph with $58283$ vertices and $133178$ edges. We run HCSE algorithm on it. For each ground-truth cluster $c$ that appears in this subgraph, we find from the resulting cluster tree an internal node that has maximum Jaccard index with $c$. Then we calculate the average Jaccard index $\overline{J}$ over all such $c$. We also calculate cost(SE) and cost(Das). The results are demonstrated in Table \ref{tab:Amazon}. HCSE performs better for $\overline{J}$ and cost(SE), while LOUVAIN performs better for cost(Das). Because of unbalance in over-fitting and under-fitting traps, HLP outperforms none of the other two algorithms for all criteria.

\begin{table}[htbp] 
	\centering
	\begin{tabular}{cccc}
		\toprule
		index     & \hspace{0.3em}HCSE     & HLP      & LOUVAIN  \\
		\midrule
		$\overline{J}$ & \textbf{0.20}  & 0.16      & 0.17\\
		cost(SE)    & \textbf{1.85E6} & 2.05E6 & 1.89E6 \\
		cost(Das)   & 5.57E8 & 3.99E8 & \textbf{3.08E8} \\
		\bottomrule
	\end{tabular}
	\smallskip
	\caption{\footnotesize Comparisons of the average Jaccard index ($\overline{J}$), cost function based on structural entropy (cost(SE)) and Dasgupta's cost function (cost(Das)).}
	\label{tab:Amazon}
\end{table}

Second, we do our experiments on a series of real networks \footnote{http://networkrepository.com/index.php} without ground truth. We compare cost(SE) and cost(Das), respectively. Since the different level numbers given by the three algorithms influence the costs seriously, that is, lower costs are obtained just due to greater heights, we only list in Table \ref{tab:real_datasets} the networks for which the three algorithms yield similar level numbers that differ by at most $1$ or $2$. It can be observed that HLP does not achieve optima for any network, while HCSE performs best w.r.t. cost(Das) for all networks, but does not outperform LOUVAIN for most networks. This is mainly due to the fact that LOUVAIN always finds no less number of hierarchies than HCSE, and the better cost probably benefits from its depth.

\begin{table}
	\centering
	\begin{tabular}{cccccccccc}
		\toprule
		Networks     & HCSE     & HLP      & LOUVAIN  \\
		\midrule
		CSphd         & 1.30E4 / \textbf{5.19E4} / 5  & 1.54E4 / 5.58E4 / 4      & \textbf{1.28E4} / 7.61E4 / 5\\
		\midrule
		fb-pages-government & 2.48E6 / \textbf{1.18E8} / 4  & 2.53E6 / 1.76E8 / 3      & \textbf{2.43E6} / 1.33E8 / 4\\
		\midrule
		email-univ & 1.16E5 / \textbf{2.20E6} / 3  & 1.46E5 / 6.14E6 / 3      & \textbf{1.14E5} / 2.20E6 / 4\\
		\midrule
		fb-messages & 1.58E5 / \textbf{4.50E6} / 4  & 1.76E5 / 8.12E6 / 3      & \textbf{1.52E5} / 4.96E6 / 4\\
		\midrule
		G22 & \textbf{5.56E5} / \textbf{2.68E7} / 4  & 6.11E5 / 4.00E7 / 3      & 5.63E5 / 2.80E7 / 5\\
		\midrule
		As20000102 & 2.64E5 / \textbf{2.36E7} / 4  & 3.62E5 / 7.63E7 / 3      & \textbf{2.42E5} / 2.42E7 / 5\\
		\midrule
		bibd-13-6 & \textbf{7.41E5} / \textbf{2.56E7} / 3  & 8.05E5 / 4.41E7 / 2      & 7.50E5 / 2.75E7 / 4\\
		\midrule
		delaunay-n10 & 4.65E4 / \textbf{3.39E5} / 4  & 4.87E4 / 3.55E5 / 4      & \textbf{4.24E4} / 4.25E5 / 5\\
		\midrule
		p2p-Gnutella05 & 9.00E5 / \textbf{1.48E8} / 3  & 1.01E6 / 2.78E8 / 3      & \textbf{8.05E5} / 1.49E8 / 5\\
		\midrule
		p2p-Gnutella08 & 5.59E5 / \textbf{5.51E7} / 4  & 6.36E5 / 1.28E8 / 4      & \textbf{4.88E5} / 6.03E7 / 5\\
		\bottomrule
	\end{tabular}
	\smallskip
	\caption{``cost(SE) / cost(Das) / $k$'' for three algorithms, where $k$ is the hierarchy number that the algorithm finds.}
	\label{tab:real_datasets}
\end{table}

\section{Conclusions and future discussions} \label{sec:conclusions}

In this paper, we investigate the hierarchical clustering problem from an information-theoretic perspective and propose a new objective function that relates to the combinatoric cost functions raised by Dasgupta \cite{dasgupta2016cost} and Cohen-Addad et al. \cite{cohen2019hierarchical}. We define the optimal $k$-level cluster tree for practical use and devise an hierarchical clustering algorithm that stratifies the sparsest level of the cluster tree recursively. This is a general framework that can be collocated with any cost function. We also propose an interpretable strategy to find the intrinsic number of levels without any hyper-parameter. The experimental results on $k$-level HSBM demonstrate that our algorithm HCSE has a great advantage in finding $k$ compared to the popular but strongly heuristic algorithms LOUVAIN and HLP. Our results on real datasets show that HCSE also achieves competitive costs compared to these two algorithms.

There are several directions that are worth further study. The first problem is about the relationship between the concavity of $g$ of the cost function and the balance of the optimal cluster tree. We have seen that for cliques, being concave is not a sufficient condition for total balance, so whether is it a necessary condition? Moreover, is there any explicit necessary and sufficient condition for total balance of the optimal cluster tree for cliques? The second problem is about approximation algorithms for both structural entropy and cost(SE). Due to the non-linear and volume-related function $g$, the proof techniques for approximation algorithms in \cite{cohen2019hierarchical} becomes unavailable. The third one is about more precise characterizations for ``natural'' hierarchical clustering whose depth is limited. Since any reasonable choice of $g$ makes the cost function achieve optimum on some binary tree, a blind pursuit of minimization of cost functions seems not to be a rational approach. More criteria in this scenario need to be studied.

\bibliographystyle{plain}
\bibliography{bibtex}

\appendix

\section{Proof of Theorem \ref{thm:SE_for_cliques}} \label{sec:proof_thm_2.2}

We restate Theorem 2.2 as follows.

\noindent\textbf{Theorem 2.2.} \emph{For any positive integer $n$, let $K_n$ be the clique of $n$ vertices with identical weight on every edge. Then a cluster tree $T$ of $K_n$ achieves minimum structural entropy if and only if $T$ is a balanced binary tree, that is, the two children clusters of each sub-tree of $T$ have difference in size at most $1$.}

Note that a balanced binary tree (BBT for abbreviation) means the tree is balanced on every internal node. Formally, for an internal node of cluster size $k$, its two sub-trees are of cluster sizes $\lfloor k/2 \rfloor$ and $\lceil k/2 \rceil$, respectively.

For cliques, since the weights of each edge are identical, we assume it safely to be $1$. By Theorem 2.1, minimizing the structural entropy is equivalent to minimizing the cost function (over $T$)
\begin{eqnarray*}
	\cost^T(G) &=& \sum\limits_{(u,v) \in E} \log \vol(u \vee v) \\
	&=& \sum\limits_{(u,v) \in E}\log\left((n-1)|u \vee v|\right) \\
	&=& \sum\limits_{(u,v) \in E}\log(n-1)+\sum\limits_{(u,v) \in E}\log|u \vee v|
\end{eqnarray*}
Since the first term in the last equation is independent of $T$, the optimization turns to minimizing the last term, which we denote by $\Gamma(T)$. Grouping all edges in $E$ by least common ancestor (LCA) of two end-points, the cost $\Gamma(T)$ can be written as the sum of the cost $\gamma$ at every internal node $N$ of $T$. Formally, for every internal node $N$, let $A,B \subseteq V$ be the leaves of the sub-trees rooted at the left and right child of $N$, respectively. We have
\begin{eqnarray*}
	\Gamma(T)&=&\sum\limits_{N}\gamma(N) \\
	\gamma(N)&=& \left(\sum\limits_{x \in A,y \in B}1 \right) \cdot \log\left(|A|+|B|\right) \\
	&=&|A|\cdot|B|\cdot\log(|A|+|B|)
\end{eqnarray*}
Now we only have to show the following lemma.

\begin{lemma} \label{lem:Gamma_equiv}
	For any positive integer $n$, a cluster tree $T$ of $K_n$ achieves minimum cost $\Gamma(T)$ if and only if $T$ is a BBT.
\end{lemma}

\begin{proof}
	
	Lemma \ref{lem:Gamma_equiv} is proved by induction on $|V|$. The key technique of tree swapping we use here is inspired by Cohen-Addad et al [4]. The basis step holds since for $|V|=2$ or $3$, the cluster tree is balanced and unique. It certainly achieves the minimum cost exclusively.
	
	Now, consider a clique $G=(V,E)$ with $n=|V| \ge 4$. Let $T_1$ be an arbitrary unbalanced cluster tree and $\lambda$ be its root. We need to prove that the cost $\Gamma(T_1)$ does not achieve the minimum.	Without loss of generality, we can safely assume the root node is unbalanced, since otherwise, we set $T_1$ to be the sub-tree that is rooted at an unbalanced node. Let $T_2$ be a tree with root $\lambda$ whose left and right sub-trees are BBTs such that they have the same sizes with the left and right sub-trees of $T_1$, respectively. Let $V_{ll}$, $V_{lr}$, $V_{rl}$ and $V_{rr}$ be the sets of nodes on the four sub-trees at the second level of $T_2$ and $n_{ll}$, $n_{lr}$, $n_{rl}$ and $n_{rr}$ denote their sizes, respectively. Our proof is also available when some of them are empty. We always assume $n_{ll} \le n_{lr}$ and $n_{rl} \ge n_{rr}$. Next, we construct $T_3$ by swapping (transplanting) $V_{lr}$ and $V_{rl}$ with each other. Finally, let $T_4$ be a tree with root $\lambda$ whose left and right sub-trees are BBTs after balancing the left and right sub-trees of $T_3$. So $T_4$ is a BBT. Then we only have to prove that $\Gamma(T_1) > \Gamma(T_4)$. Note that the strict ``$>$'' is necessary since we need to negate all unbalanced cluster trees.
	
	Then we show that the transformation process that consists of the above three steps makes the cost decrease step by step. Formally,
	\begin{itemize}
		\item[(a)] $T_1$ to $T_2$. The sub-trees of $T_1$ become BBTs in $T_2$. Since the number of edges whose end-points treat the root as LCA is the same, by induction we have $\Gamma(T_1) \ge \Gamma(T_2)$.
		\item[(b)] $T_2$ to $T_3$. We will show that $\Gamma(T_2) > \Gamma(T_3)$ in Lemma \ref{lem:T2_T3_transplant}.
		\item[(c)] $T_3$ to $T_4$. The sub-trees of $T_3$ become BBTs in $T_4$. For the same reason as (a), we have $\Gamma(T_3) \ge \Gamma(T_4)$. 
	\end{itemize}
	This transformation process is illustrated in Figure \ref{fig:transformation}. Putting them together, we get $\Gamma(T_1) > \Gamma(T_4)$ and Lemma \ref{lem:Gamma_equiv} follows.
	
	\begin{figure}[h]
		\centering
		\includegraphics[width=1\textwidth]{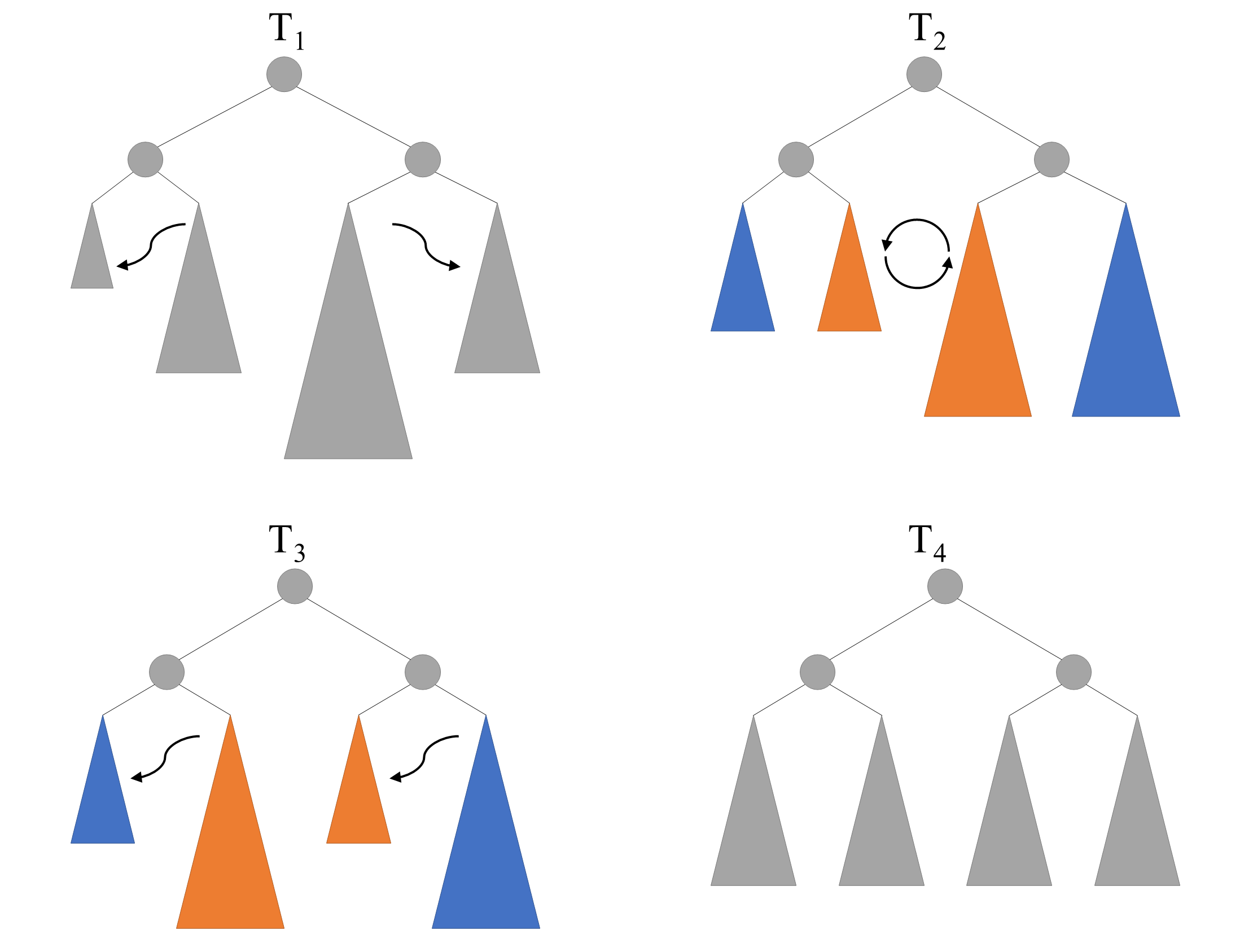}
		\caption{Illustration of transformation from $T_1$ to $T_4$.}
		\label{fig:transformation}
	\end{figure}
\end{proof}

\begin{lemma} \label{lem:T2_T3_transplant}
	After swapping $V_{lr}$ and $V_{rl}$, we obtain $T_3$ from $T_2$, for which $\Gamma(T_2) > \Gamma(T_3)$. 
\end{lemma}

\begin{proof}
	We only need to consider the changes in cost of three nodes: root and its left and right children, since the cost contributed by each of the remaining nodes does not change after swapping. Ignoring the unchanged costs, define
	\begin{eqnarray*}
		\cost(T_2) &=& n_l n_r \log n+n_{ll} n_{lr} \log{n_l}+n_{rl} n_{rr} \log{n_r} \\
		&=& n_ln_r \log n+\left\lfloor \frac{n_l}{2} \right\rfloor \left\lceil \frac{n_l}{2} \right\rceil  \log{n_l}+\left\lceil \frac{n_r}{2} \right\rceil \left\lfloor \frac{n_r}{2} \right\rfloor \log{n_r},
	\end{eqnarray*}
	where $n_l = n_{ll} + n_{lr}$, $n_r = n_{rl} + n_{rr}$. Both of them are at least $1$. Similarly, define
	\begin{eqnarray*}
		\cost(T_3) &=& (n_{ll}+n_{rl})(n_{lr}+n_{rr}) \log n+n_{ll}n_{rl} \log{(n_{ll}+n_{rl})}+n_{lr}n_{rr} \log{(n_{lr}+n_{rr})} \\
		&=& \left\lfloor \frac{n}{2} \right\rfloor \left\lceil \frac{n}{2} \right\rceil \log n+\left\lfloor \frac{n_l}{2} \right\rfloor \left\lceil \frac{n_r}{2} \right\rceil \log\left(\left\lfloor \frac{n_l}{2} \right\rfloor+\left\lceil \frac{n_r}{2} \right\rceil\right)+\left\lceil \frac{n_l}{2} \right\rceil \left\lfloor \frac{n_r}{2} \right\rfloor \log \left(\left\lceil \frac{n_l}{2} \right\rceil+\left\lfloor \frac{n_r}{2} \right\rfloor \right) \\
	\end{eqnarray*}
	Denote
	\begin{eqnarray} \label{eqn:Delta}
		\Delta &=& \Gamma(T_2) - \Gamma(T_3) \nonumber \\
		&=& \cost(T_2) - \cost(T_3) \nonumber \\
		&=&\left\lfloor \frac{n_l}{2} \right\rfloor \left\lceil \frac{n_l}{2} \right\rceil \log\left(\frac{n_l}{n}\right)+\left\lceil \frac{n_r}{2} \right\rceil \left\lfloor \frac{n_r}{2} \right\rfloor \log\left(\frac{n_r}{n}\right) \nonumber \\
		& &-\left\lfloor \frac{n_l}{2} \right\rfloor \left\lceil \frac{n_r}{2} \right\rceil \log\left(\frac{\left\lfloor \frac{n_l}{2} \right\rfloor+\left\lceil \frac{n_r}{2} \right\rceil}{n}\right)-\left\lceil \frac{n_l}{2} \right\rceil \left\lfloor \frac{n_r}{2} \right\rfloor \log\left(\frac{\left\lceil \frac{n_l}{2} \right\rceil+\left\lfloor \frac{n_r}{2} \right\rfloor}{n}\right)
	\end{eqnarray}
	So we only have to show that $\Delta > 0$. We consider the following three cases according to the odevity of $n_l$ and $n_r$.
	\begin{itemize}
		\item[\textbf{Case $1$}:] $n_l$ and $n_r$ are even.
		\item[\textbf{Case $2$}:] $n_l$ and $n_r$ are odd.
		\item[\textbf{Case $3$}:] $n_l$ is odd while $n_r$ is even.
	\end{itemize}
	The case that $n_l$ is even while $n_r$ is odd is symmetric to \textbf{Case $3$}.
	
	For \textbf{Case $1$}, if both $n_l$ and $n_r$ are even, then notations of rounding in Eq. (\ref{eqn:Delta}) can be removed and $\Delta$ can be simplified as
	\begin{eqnarray*}
		\Delta=\frac{n_l^2}{4} \log\left(\frac{n_l}{n}\right)+\frac{n_r^2}{4} \log\left(\frac{n_r}{n}\right)+\frac{n_l n_r}{2}.
	\end{eqnarray*}
	Let $p=n_l/n,q=n_r/n$, and so $p+q=1$. Recall that $T_1$ is unbalanced on the root $\lambda$, so is $T_2$. Thus $p\neq q$. Multiplying by $\frac{4}{n^2}$ on both sides, we only have to prove that
	\[
	p^2 \log p+q^2 \log q+2pq > 0.
	\]	
	That is,
	$$\frac{p}{q} \log p+\frac{q}{p} \log q+2 > 0.$$			
	Let $g(x)=\frac{x}{1-x} \log x$. Then we only need to show that $g(p)+g(q)+2>0$ when $p\neq q$. Since
	\begin{eqnarray*}
		g'(x) &=& \frac{(1-x)+\ln x}{\ln 2 \cdot (1-x)^2}, \\
		g''(x) &=& -\frac{x^2-2x\ln x-1}{\ln 2 \cdot x(1-x)^3}.
	\end{eqnarray*}
	It is easy to check that $g''(x)>0$ when $0 < x < 1$. So $g(x)$ is strictly convex in the interval $(0,1)$. Since $p\neq q$,
	$$g(p)+g(q) > 2g\left(\frac{p+q}{2}\right) = -2.$$
	Thus $\Delta>0$ holds.
	
	For \textbf{Case $2$}, if both $n_l$ and $n_r$ are odd, then $\Delta$ can be split into two parts $\Delta=\Delta_1+\Delta_2$, in which
	\begin{eqnarray*}
		\Delta_1 &=& \frac{n_l^2}{4} \log\left(\frac{n_l}{n}\right)+\frac{n_r^2}{4} \log\left(\frac{n_r}{n}\right)+\frac{n_l n_r}{2} \\
		\Delta_2 &=& -\frac{1}{4} \log\left(\frac{n_l}{n}\right)-\frac{1}{4} \log\left(\frac{n_r}{n}\right)-\frac{1}{2}
	\end{eqnarray*}
	Since we have shown that $\Delta_1>0$, if we can prove $\Delta_2 \ge 0$, then the lemma will hold for \textbf{Case $2$}. Due to the convexity of logarithmic function, this holds clearly since
	\begin{eqnarray*}
		2 \log\left(\frac{n}{2}\right) \ge \log n_l + \log n_r.
	\end{eqnarray*}

	For \textbf{Case $3$}, if $n_l$ is odd while $n_r$ is even,
	\begin{eqnarray*}
		\Delta=\frac{n_l^2-1}{4} \log\left(\frac{n_l}{n}\right)+\frac{n_r^2}{4} \log\left(\frac{n_r}{n}\right)-\left[\frac{(n_l-1)n_r}{4} \log\left(\frac{n-1}{2n}\right)+\frac{(n_l+1)n_r}{4} \log\left(\frac{n+1}{2n}\right)\right].
	\end{eqnarray*}
	Multiplying the above equation by $4\ln 2$, without changing its sign, yields
	\begin{eqnarray*}
		(4\ln 2) \Delta = (n_l^2-1)\ln\left(\frac{n_l}{n}\right)+n_r^2\ln\left(\frac{n_r}{n}\right)-\left[(n_l-1)n_r\ln\left(\frac{n-1}{2n}\right)+(n_l+1)n_r\ln\left(\frac{n+1}{2n}\right)\right]
	\end{eqnarray*}	
	Splitting the right hand side into two parts,
	\begin{eqnarray*}
		A &=& n_l^2\ln\left(\frac{n_l}{n}\right)+n_r^2\ln\left(\frac{n_r}{n}\right)+2n_ln_r\ln2 \\
		B &=& -\ln\left(\frac{n_l}{n}\right)-(n_l+1)n_r\ln\left(1+\frac{1}{n}\right)-(n_l-1)n_r\ln\left(1-\frac{1}{n}\right)
	\end{eqnarray*}
	Since $n$ is odd and the root $\lambda$ of $T_2$ is unbalanced, we only need to consider the case that $n_l=(n-i)/2$, $n_r=(n+i)/2$ (Note that $n_l$ and $n_r$ are symmetric. So if $(n-i)/2$ is even, exchange $n_l$ and $n_r$), where both $n$ and $i$ are odd satisfying $n > i \ge 3$. Next we show that in this case, $A \ge \ln(1/5)+4^2\ln(4/5)+2 \cdot 4\ln2$ and $B > \ln2-3/4-(2/3) \cdot (1/5^2)$. By calculation, $\Delta=A+B > 0$ for \textbf{Case $3$}.		
	
	\begin{claim} \label{clm:A}
		$A \ge \ln(1/5)+4^2\ln(4/5)+2 \cdot 4\ln2$ for odd integers $n > i \ge 3$.
	\end{claim}
	
	\begin{proof}
		Substituting $n_l=(n-i)/2$, $n_r=(n+i)/2$ into the $A$ yields
		\begin{eqnarray*}
			A = C(n,i) \triangleq \left(\frac{n-i}{2}\right)^2\ln\left(\frac{n-i}{2n}\right)+\left(\frac{n+i}{2}\right)^2\ln\left(\frac{n+i}{2n}\right)+2\cdot\frac{n-i}{2}\cdot\frac{n+i}{2}\ln2.
		\end{eqnarray*}
		Treat $n$ as a continuous variable, we have
		\begin{eqnarray*}
			\frac{\partial C(n,i)}{\partial n}=\frac{1}{2} \left[{(n+i)\ln\left(1+\frac{i}{n}\right)+(n-i)\ln\left(1-\frac{i}{n}\right)-\frac{i^2}{n}}\right]
		\end{eqnarray*}
		Multiplying the above equation by $2/n$ and setting $x = i/n$ yields
		\begin{eqnarray*}
			f(x) &\triangleq& (1+x)\ln(1+x)+(1-x)\ln(1-x)-x^2, \\
			f'(x) &=& \ln(1+x)-\ln(1-x)-2x, \\
			f''(x) &=& \frac{2x^2}{1-x^2}.
		\end{eqnarray*}
		It is easy to check that $f(0) = 0$ and $f'(0) = 0$. When $0 < x < 1$,  $f''(x) > 0$. Thus $f'(x) > 0$ and $f(x) > 0$. This means that $\partial C(n,i)/\partial n > 0$ for all $n>0$. So $C(n,i) \ge C(i+2,i)$ for $n \ge i+2$ (When $i$ is fixed, the minimum value of $n$ can be taken to $i + 2$, which makes $n_l=(n-i)/2$ and $n_r=(n+i)/2$ integral). The curves of $C(n,i)$ for varying $i$ are plotted in Figure \ref{fig:C(n,i)}.
		
		\begin{figure}[htb]
			\centering
			\includegraphics[width=0.6\textwidth]{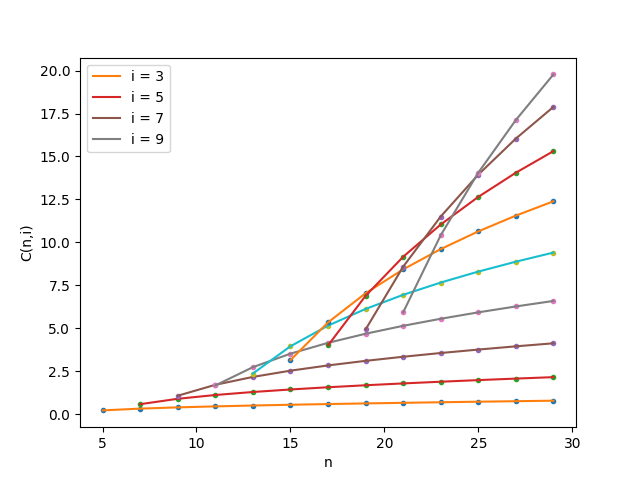}
			\caption{Functions $C(n,i)$}
			\label{fig:C(n,i)}
		\end{figure}
		
		When $n = i+2$, we get $n_l=(n-i)/2=1$ and $n_r=(n+i)/2=n-1$. Substituting them into $A$ yields
		\begin{eqnarray*}
			D(n) &\triangleq& \ln\left(\frac{1}{n}\right)+(n-1)^2\ln\left(1-\frac{1}{n}\right)+2(n-1)\ln2, \\
			\frac{d D}{dn} &=& 1-\frac{2}{n}+2\ln2+2(n-1)\ln\left(1-\frac{1}{n}\right).
		\end{eqnarray*}
		When $n>2$, it is easy to check that $d D/d n>0$. So the minimum value of $d(n)$, which is also the minimum value of $C(i+2,i)$, is achieved at $n=i+2=5$. So $A=C(n,i) \ge C(i+2,i) \ge C(5,3)=\ln(1/5)+4^2\ln(4/5)+2 \cdot 4\ln2$.
	\end{proof}
	
	\begin{claim} \label{clm:B}
		$B > \ln2-3/4-(2/3) \cdot (1/5^2)$.
	\end{claim}
	
	\begin{proof}
		Due to the facts that
		\begin{eqnarray*}
			\ln\left(1+\frac{1}{n}\right) &<& \frac{1}{n}-\frac{1}{2n^2}+\frac{1}{3n^3}, \\
			\ln\left(1-\frac{1}{n}\right) &<& -\frac{1}{n}-\frac{1}{2n^2}-\frac{1}{3n^3},
		\end{eqnarray*}
		we have
		\begin{eqnarray*}
			B &=& -\ln\left(\frac{n_l}{n}\right)-(n_l+1)n_r\ln\left(1+\frac{1}{n}\right)-(n_l-1)n_r\ln\left(1-\frac{1}{n}\right) \\
			&>& -\ln\left(\frac{n_l}{n}\right)+\frac{n_ln_r}{n^2}-\frac{2n_r}{n}-\frac{2n_r}{3n^3} \\
			&>& -\ln\left(\frac{n_l}{n}\right)+\frac{n_ln_r}{n^2}-\frac{2n_r}{n}-\frac{2}{3n^2}.
		\end{eqnarray*}
		Let $\alpha=n_l/n$, then
		\begin{eqnarray*}
			B &>& -\ln\alpha+\alpha(1-\alpha)-2(1-\alpha)-\frac{2}{3n^2} \\
			&\geq& \ln2-\frac{3}{4}-\frac{2}{3n^2}.
		\end{eqnarray*}
		When $n \ge 5$, $B > \ln2-3/4-(2/3) \cdot (1/5^2)$.
	\end{proof}
	Combining Claims \ref{clm:A} and \ref{clm:B}, Lemma \ref{lem:T2_T3_transplant} follows.
\end{proof}
This completes the proof of Theorem \ref{thm:SE_for_cliques}.

\end{document}